\documentclass[nohyperref]{article}

\usepackage{microtype}
\usepackage{graphicx}
\usepackage{subfigure}
\usepackage{booktabs} %
\usepackage{overpic}

\usepackage{hyperref}

\usepackage[accepted]{arxiv}

\usepackage{amsmath}
\usepackage{amssymb}
\usepackage{mathtools}
\usepackage{amsthm}

\usepackage[capitalize,noabbrev]{cleveref}

\theoremstyle{plain}

\theoremstyle{definition}

\theoremstyle{remark}

\newtheorem{observation}{Observation}[section]

\usepackage{algorithm}
\usepackage[noend]{algpseudocode}
\usepackage{wrapfig}

\newcommand{\approach}{fitness monotonic execution}

\newcommand\N{\mathbb{N}}
\newcommand\R{\mathbb{R}}

\icmltitlerunning{Eliminating Meta Optimization Through Self-Referential Meta Learning}

\begin{document}

\twocolumn[
\icmltitle{Eliminating Meta Optimization Through Self-Referential Meta Learning}

\icmlsetsymbol{equal}{*}

\begin{icmlauthorlist}
\icmlauthor{Louis Kirsch}{idsia}
\icmlauthor{Jürgen Schmidhuber}{idsia,kaust}
\end{icmlauthorlist}

\icmlaffiliation{idsia}{The Swiss AI Lab IDSIA, USI, SUPSI}
\icmlaffiliation{kaust}{King Abdullah University of Science and Technology}

\icmlcorrespondingauthor{Louis Kirsch}{louis@idsia.ch}

\icmlkeywords{Machine Learning, ICML, Reinforcement Learning, Meta Learning, Self-Referential, Self-Improvement}

\vskip 0.3in
]

\printAffiliationsAndNotice{}  %

\begin{abstract}
Meta Learning automates the search for learning algorithms.
At the same time, it creates a dependency on human engineering on the meta-level, where meta learning algorithms need to be designed.
In this paper, we investigate self-referential meta learning systems that modify themselves without the need for explicit meta optimization.
We discuss the relationship of such systems to in-context and memory-based meta learning and show that self-referential neural networks require functionality to be reused in the form of parameter sharing.
Finally, we propose {\approach} (FME), a simple approach to avoid explicit meta optimization.
A neural network self-modifies to solve bandit and classic control tasks, improves its self-modifications, and learns how to learn, purely by assigning more computational resources to better performing solutions.
\end{abstract}

\section{Introduction}

Machine learning is the process of deriving models and behavior from data or environment interaction using human-engineered learning algorithms.
Meta learning takes this process to the meta-level:
Its goal is to derive the learning algorithms themselves automatically as well~\citep{schmidhuber1987evolutionary,hochreiter2001learning,wang2016learning,duan2016rl,finn2017model,flennerhag2019meta,kirsch2019improving,kirsch2020meta}.
Unfortunately, this usually creates a dependency on human engineering on the meta-level, where researchers now have to design meta learning algorithms.
In this paper, we investigate methods for neural self-referential meta learning~\citep{schmidhuber1993self,schmidhuber1997shifting}.
In particular, we seek a process of self-improvement that reduces our reliance on human engineering to the largest extent possible.

A central piece in the discussion of self-referential meta learning is the self-referential neural architecture~\citep{schmidhuber1993self}.
Such an architecture allows for modifying not just some memory to improve future behavior, but also all of its own neural weights.
This enables self-modification and self-improvement allowing the architecture to learn, meta-learn, meta-meta-learn, etc.
We point out that
(1) in order to construct systems that can change all their parameters (or variables more generally), parts of the computational graph need to be reused.
This is done in the form of parameter (variable) sharing.
(2) We discover that memory-based architectures (in-context learners) are capable of similar self-improvement.
There is a representational equivalence between neural networks with memory and self-referential architectures.
Despite this, we show that self-referential architectures are useful in the absence of meta optimization.

Finally, we propose {\approach} (FME), a simple approach to avoid explicit meta optimization.
Instead of proposing changes to the model or learning algorithm explicitly, all changes to the model are self-modifications and the resulting solutions are selected for execution more frequently the better their performance.
We empirically demonstrate FME with a neural network that self-modifies  to solve bandit and classic control tasks, improves its self-modifications, and learns how to learn.

\section{Background}

\subsection{Self-referential Architectures}
A key requirement for a self-improving system is that it can make self-modifications such that it can change its behavior and learning arbitrarily.
One previously suggested way of achieving this is to bring all variables in a neural network under control of the network itself~\citep{schmidhuber1993self}.
This is referred to as a self-referential neural architecture.
Compare this to a conventional neural network where there is a subset of variables (called the weights or parameters) that are only updated by a fixed learning algorithm (such as backpropagation).
This entails that part of the (meta-)learning behavior is fixed and needs to be defined by the researcher.

In contrast, self-referential architectures control all variables.
This includes activations (conventionally updated by the neural network itself), weights (conventionally updated by a learning algorithm), meta weights, etc.
In this section, we discuss necessary conditions on such a self-referential architecture and possible implementations.

\paragraph{Notation}
In the remainder of this paper, we denote external inputs to the neural network as $x \in \R^{N_x}$ (such as observations and rewards in Reinforcement Learning, or error signals in supervised learning), outputs as $y \in \R^{N_y}$ (e.g. actions in Reinforcement Learning), and the parameters of the neural network as $\theta$.
Further, we denote time-varying variables (memory) as $h \in \R^{N_h}$ (such as the hidden state of an RNN).
We summarize all variables in a neural network as $\phi = \{\theta, h, y\}$.

\paragraph{Necessary Conditions}
A self-referential architecture $\phi \leftarrow g_\phi(x)$ is described by a connected computational graph for the function $g$ that 
has variables $\phi = \{\theta, h, y\}$ (one node per scalar).
The network controls all of its variables in the sense that any element(s) of $\phi$ can be changed by the network itself.
This blurs the distinction between activations (memory) $h$ and weights $\theta$.
Computational graphs that fulfill this definition are required to have a certain structure.
At least some variables need to be reused in multiple operations (node out-degree $> 1$).
We refer to this as variable sharing, %
a generalization of weight sharing~\citep{fukushima1979neural} extending beyond classical weights.

Consider a square dense weight matrix.
It consists of $N^2$ weights and $N$ activations.
While the activations are time-varying, the weights are source nodes in the computational graph and cannot be directly updated by actions of the network itself.
To change that, we need to derive $N^2$ variables from $N$ time-varying variables.
This can only be done by reusing some of the same $N$ time-varying variables in multiple operations, generating the $N^2$ variables.

\begin{observation}
Variable sharing in self-referential systems.
Assuming a connected computational graph, an architecture that updates all its variables $\phi \in \R^{N_\phi}$ in iteration $t$ needs to reuse elements of $\phi$ multiple times in the computational graph to generate $\phi_{t+1} \in \R^{N_\phi}$ from $\phi_t \in \R^{N_\phi}$.
\end{observation}
\begin{proof}
As there are no more elements in $\phi_t$ than there are in $\phi_{t+1}$, any operation generating an element in $\phi_{t+1}$ that makes use of more than one element in $\phi_t$ needs to reuse an element already in use by a different operation.
\end{proof}

\paragraph{Implementations}
Under the previous constraints, various implementations for self-referential neural architectures are conceivable.
\citet{schmidhuber1993self} assigns an address to each weight such that the network outputs can be used to attend to weights and both read and write their values.
Instead of updating one weight at a time, the fast weight programmers (FWPs) of 1992-93~\citep{Schmidhuber:92ncfastweights,Schmidhuber:93ratioicann} are networks that learn to generate key and value patterns to rapidly change many fast weights simultaneously (although not all FWPs are fully self-referential).
Outer products between activations (a type of variable sharing) are used to derive $M * N$ variables, $M,N \in \N$, from $M + N$ variables.
This allows for updating all the weights of a neural network layer by its own activations~\citep{irie2021a}.
Alternatively, a coordinate-wise mechanism may generate all updates continuously as a function of the weight address~\citep{dambrosio2007novel}.
Other works have used multiple RNNs with shared weights and messaging passing between those to increase the number of time-varying variables $h$ arbitrarily while keeping the number of parameters $\theta$ constant~\citep{rosa2019badger,kirsch2020meta}.
Such systems can be made self-referential by using a subset of $h$ to update parameters $\theta$.

\subsection{Expressivity of Memory and Self-referential Architectures}

In this section, we show that self-referential architectures do not have a fundamental representational advantage over memory-based architectures (in-context learners) when the free (initial) variables are meta-optimized using a human engineered learning algorithm.

We defined self-referential architectures $\phi \leftarrow g_\phi(x)$ as those that can update all their variables $\phi$ in the computational graph.
For notational purposes, we can express $y$ as the explicit output $\phi, y \leftarrow g_\phi(x)$.
Compare this to a memory architecture such as a recurrent neural network $h, y \leftarrow f_\theta(h, x)$ parameterized by $\theta$ where $h$ corresponds to its hidden state (memory).
Can the self-referential architecture represent any functions that the memory architecture can not?
A commonly used intuition~\citep{schmidhuber1993self} is that self-referential architectures are self-modifying, in that they change their own weights, affecting not only their outputs and current weights but also future weight changes through $g_\phi$.
These architectures can thus not only learn, but also meta-learn, meta-meta-learn, etc.
While memory architectures do not update their weights, they can also be self-modifying.
Changes in memory $h$ affect the output directly, but also the effective function $f_\theta$ by modifying its input $h$, in turn determining future changes to $h$.
More generally, we can use in-context learners (memory-based architectures) to simulate or approximate self-referential systems:

\begin{observation}
For any self-referential architecture $\phi, y \leftarrow g_\phi(x)$ and some initial $\phi_0$ we can find a memory architecture $h, y \leftarrow f_\theta(h, x)$, $\theta$, and initial $h_0$ such that for any sequence of $x_{1:T}$ we have $\hat f(x_{1:T}) = \hat g(x_{1:T})$ where $\hat f$ and $\hat g$ are the unrolls returning $y_{1:T}$ of $f$ and $g$ respectively.
\end{observation}
\begin{proof}
We construct an emulator $f_\theta$ (a sufficiently large neural network parameterized by $\theta$) that stores $\phi$ in $h$ and sets $\theta, h_0$ such that at each step $t \in \N$ it performs the same computation as $g_\phi$ by taking $\phi_t$ from $h_t$, computing $\phi_{t+1}$, and storing it in $h_{t+1}$.
\end{proof}

\begin{figure}
    \centering
    \begin{overpic}[width=\linewidth]{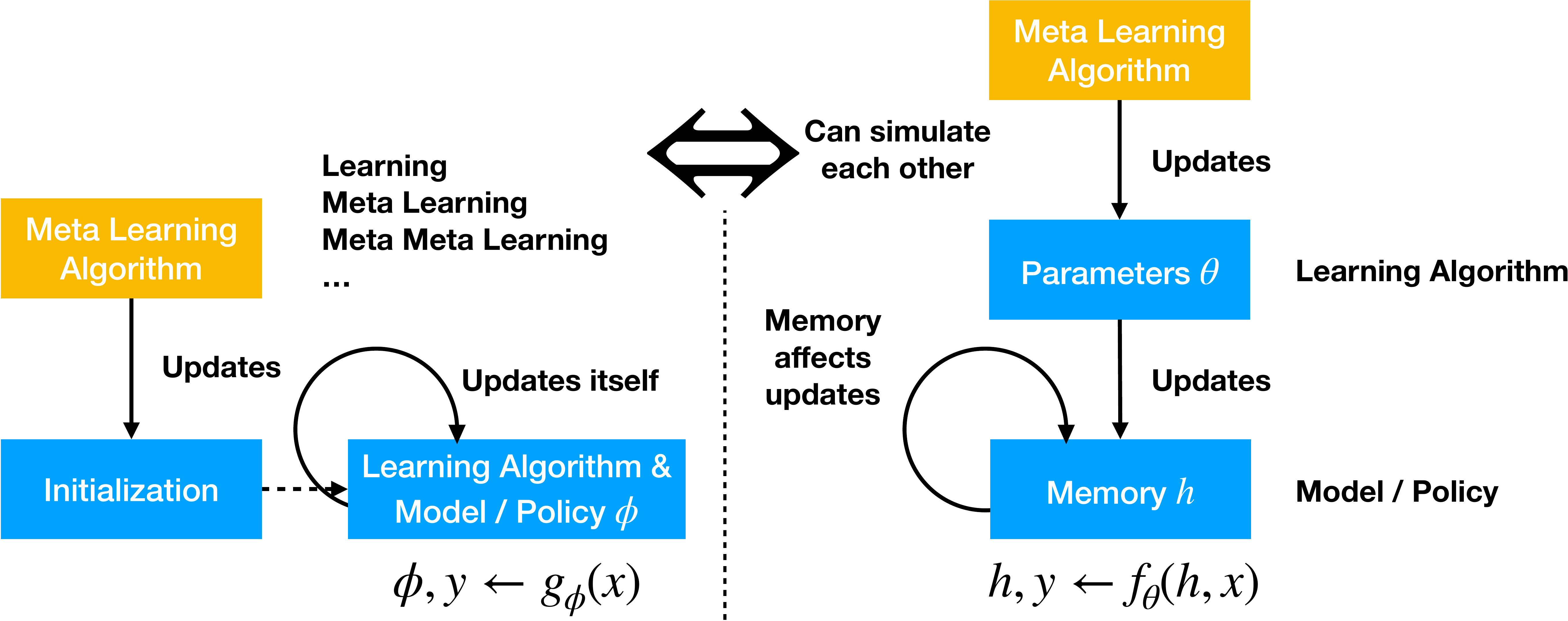}
        \put (0,35) {\textbf{\small(a)}}
        \put (47,35) {\textbf{\small(b)}}
    \end{overpic}
    \caption{
    \textbf{Self-referential (a) and memory-based (b) architectures (in-context learners) can represent the same function class and both require meta optimization.}
    Self-referential networks can directly update all of their weights, whereas memory-based networks can self-modify through memory updates.
    The former require meta optimization of the weight initialization $\phi_0$, the latter require the initial memory $h_0$ and weights $\theta$.
    }
    \label{fig:architectures}
\end{figure}

Furthermore, any memory-based architecture can be represented by a self-referential architecture where a subset of variables is updated by the identity function.
In conclusion, the function class that can be represented by self-referential architectures is equivalent to memory architectures, given sufficiently rich parameterizations.
This is visualized in \autoref{fig:architectures}.
For both memory architectures and self-referential architectures the same question arises:
How do we set the free variables $\theta$, $h_0$, or $\phi_0$?
If these free variables are directly optimized (eg by following the gradient of some objective), from a representability perspective there is no advantage in self-referential architectures.
Orthogonal to this, the chosen architecture (whether memory or self-referential) may have varying optimizability or regularizing benefits due to e.g. sparsity in the computational graph, multiplicative interactions, or variable sharing.
In the following, we discuss how self-referential architectures are relevant in the absence of direct meta optimization.%

\section{Method: Fitness Monotonic Execution}

Both in the case of self-referential and memory architectures the free (initial) variables need to be found.
Here we propose a method, \textit{\approach} (FME), that avoids explicit meta optimization of these free variables.
Instead of modifying $\phi$ directly using a human-engineered learning algorithm, we simply select between different configurations of $\phi$ that are generated using self-modifications.
In particular, through interactions with the environment we continuously add new solutions to a set of $\Phi = \{\phi_i\}$.
Computation time is distributed across solutions monotonic in their performance, i.e. better performing solutions are executed longer (or are selected for execution more frequently).
This can be formalized as a pmf $p(\phi)$ that assigns each solution  $\phi \in \Phi$ a probability for being executed at any given time-step based on its average reward $\frac{R(\phi)}{\Delta t}$ relative to other solutions (where $\Delta t$ is the solution's total lifetime).
This can be further normalized by the frequency of solutions (e.g. the performance of solutions determines the probability, not their identity) to prevent many bad performing solutions to dominate over few good ones.
As a special case, $p$ may put all probability mass on the current best solution, greedily selecting for improvement.
See \autoref{alg:fitness_monotonic_execution} and \autoref{fig:fme_vis} for a full description.

Note that in this scheme, memory architectures are now inadequate.
If there were any variables that are not modifiable, their value could not be determined through self-modifications.
As we do not use any human-engineered optimization process, their value would be undefined.
Thus, self-referential architectures are required.\footnote{Although, one might set these non-modifiable variables to some fixed (random) value that still allows the resulting network to self-modify sufficiently to implement any desirable policy, learner, meta-learner, etc. We do not further explore this here.}

\begin{figure}
    \centering
    \includegraphics[width=0.8\columnwidth]{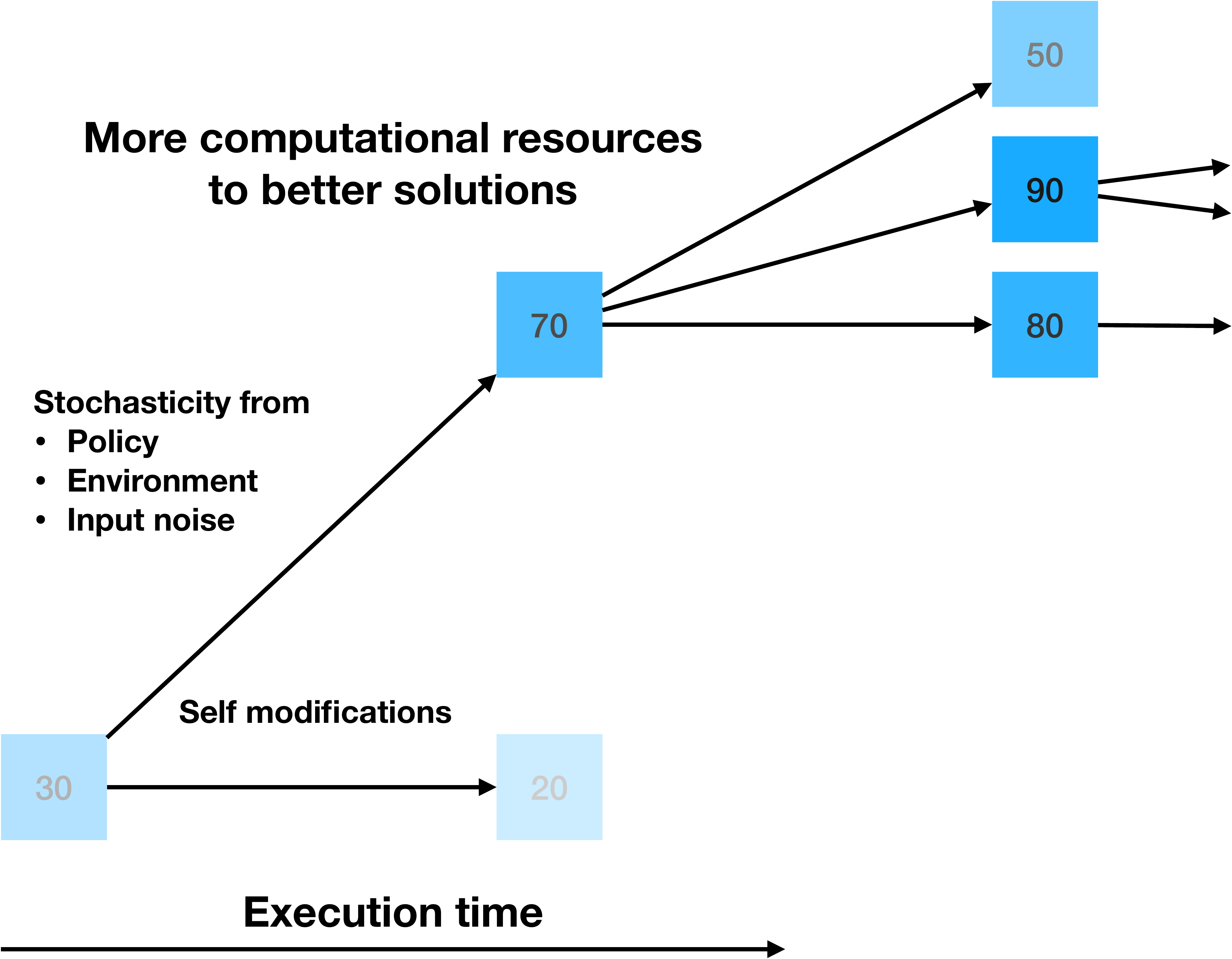}
    \caption{
    \textbf{In fitness monotonic execution (FME) the self-referential system self-modifies from a random initialization. Better performing solutions are executed longer (or selected more frequently).}
    In effect, solutions self-modify their behavior, learning, meta learning, meta meta learning, etc without a fixed scheme for meta optimization prescribing how parameters are updated.
    }
    \label{fig:fme_vis}
\end{figure}

\begin{algorithm}
    \centering	
    \begin{algorithmic}	
        \Require Initial solution(s) $\Phi = \{\phi_i\}$, self-referential architecture $g_\phi$, probability $p(\phi)$, an RL environment $E$
        \While{forever}
            \State $\phi \sim p(\phi)$ where $\phi \in \Phi$ \Comment{\parbox[t]{.45\linewidth}{Sample next solution to execute, monotonic in its performance}}
            \State $\phi, y_{1:L} \leftarrow g^L_\phi(x_{1:L})$ \Comment{\parbox[t]{.51\linewidth}{Execute $g$ for $L$ steps with $x_{1:L}$ from the environment $E$ including a feedback signal}}
            \State $\Phi \leftarrow \Phi \cup \{\phi\}$ \Comment{Add new $\phi$ to $\Phi$}
        \EndWhile
    \end{algorithmic}	
    \caption{Fitness monotonic execution}\label{alg:fitness_monotonic_execution}
\end{algorithm}

\paragraph{Least-recently-used Buffer}
To limit the number of solutions we need to store, we implement {\approach} with a least-recently-used (LRU) buffer.
It consists of $m$ buckets where each bucket holds recent solutions in a specific performance range.
Solutions from buckets with higher performance are sampled exponentially more frequently.

\paragraph{Outer-product-based Architecture}
For the self-referential neural network architecture, we chose an outer-product mechanism adapted from prior work~\citep{irie2021a}.
By applying a weight matrix $W_{t-1} \in \R^{N_x \times (N_y + 2N_x + 4)}$ to some input $x_t \in \R^{N_x}$ we generate the output $y_t \in \R^{N_y}$, key $k_t \in \R^{N_x}$, query $q_t \in \R^{N_x}$, and a learning rate $\beta_t \in \R^4$.
Using an outer-product, the key and query generate an update to the weight matrix $W_{t-1}$, obtaining $W_t$:
\begin{align}
    y_t, k_t, q_t, \beta_t &= W_{t-1}\psi(x_t) \\
    \bar v_t &= W_{t-1}\psi(k_t) \\
    v_t &= W_{t-1}\psi(q_t) \\
    W_t &= W_{t-1} + \sigma(\beta_t)(\psi(v_t) - \psi(\bar v_t)) \otimes \psi(k_t)
\end{align}
where $\psi$ is the tanh activation, $\sigma$ is the sigmoid function, and $\otimes$ is the outer product.
The learning rate $\beta_t \in \R^4$ controls the rate of update to the four parts generating $y,k,q,\beta$.
We stack multiple such self-referential layers.

\section{Experiments}
We empirically investigate several questions:
Firstly, starting with a randomly initialized solution, can the network modify itself to solve a bandit task?
We compare this to a hill climbing strategy.
Secondly, how do self-modifications compare when solving a markov decision process?
Thirdly, given a bandit task that is non-stationary, can the network learn to modify itself based on the reward it receives as input?
Refer to \autoref{app:exp_details} for implementation details.

\paragraph{Learning a Bandit Policy}
The first question we investigate is whether a randomly initialized self-referential architecture is capable of making self-modifications that lead to a useful policy for a given task.
We test this on a simple 2-armed bandit where one arm gives payouts (rewards) of $1$ and the other $0$.
From \autoref{fig:bandit} (left) we observe that after around $40$ self-modifications and selections a solution is found that always selects the arm with a higher payoff.
We compare this to hill climbing with a variance-tuned Gaussian noise on the network parameters.
Hill-climbing is a natural baseline, as it proposes new solutions using a fixed scheme (instead of self-modifications) and then chooses the best solution.
Thus, it validates the usefulness of self-modifications.
To match this baseline closely to FME, we simply replace self-modifications with fixed variance-tuned Gaussian noise.
We observe that in this simple environment, {\approach} is as effective as hill climbing to find an optimal solution to this bandit problem.

\begin{figure}
    \centering
    \includegraphics[width=0.49\columnwidth]{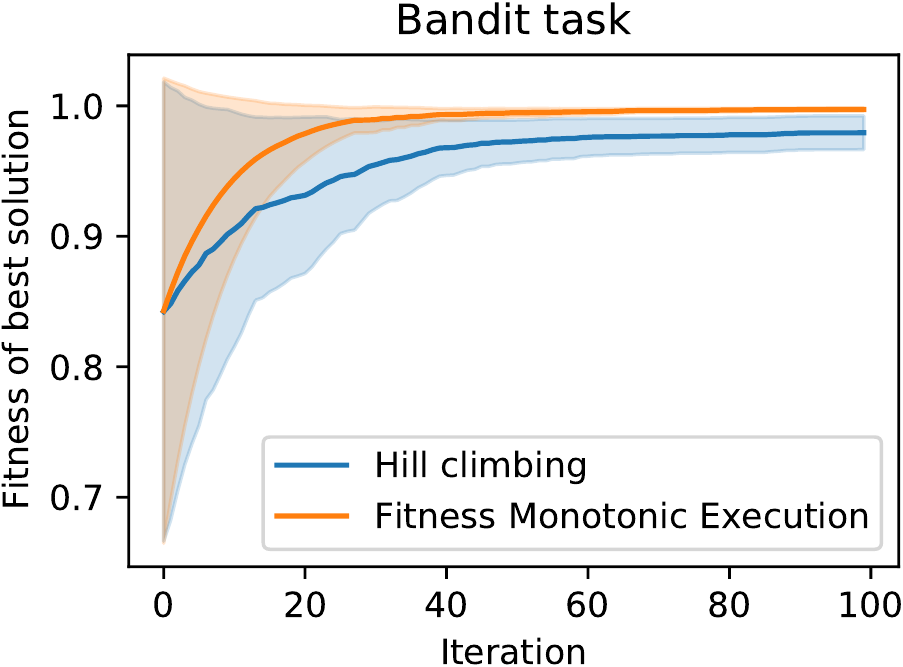}
    \includegraphics[width=0.49\columnwidth]{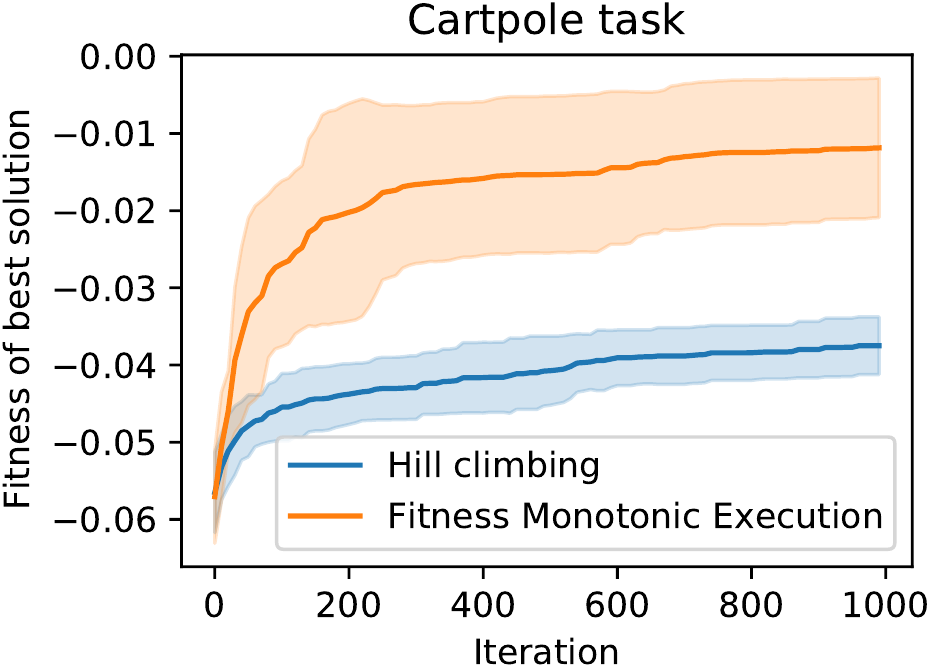}
    \caption{A randomly initialized self-referential architecture makes modifications to itself to solve a two-armed bandit problem (left).
    On a Cartpole task (right) the self-modifications not only directly improve the policy, but also improve future improvements, resulting in faster learning compared to hill climbing.
    The found policy balances the pole for about 100 steps.
    Standard deviations are shown for 5 seeds.}
    \label{fig:bandit}
\end{figure}

\paragraph{Cartpole}
Next, we increase the difficulty of the policy to be found by running a self-referential network on the Cartpole task (\autoref{fig:bandit}, right).
We observe that reaching a good performing policy takes significantly more self-modifications and selections.
At the same time, a simple hill climbing strategy (with tuned noise) fails at improving the policy at the same rate as the self-modifying architecture.
This suggests that we are not only selecting for good policies but also strategies for self-modification that lead to policy improvement in the future.

\begin{figure}
    \centering
    \includegraphics[width=0.7\columnwidth]{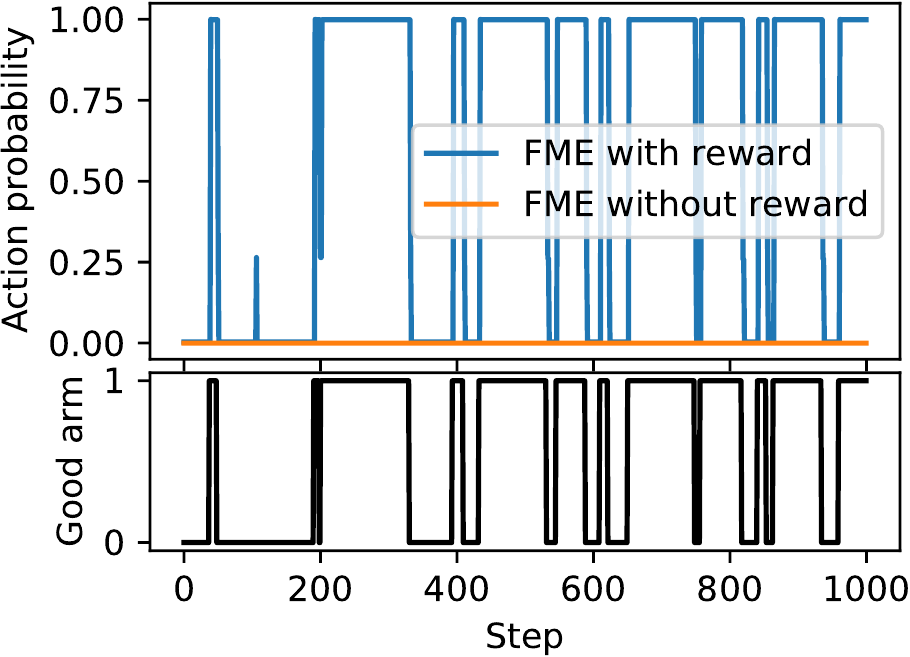}
    \caption{Given the reward as input, self-modifications enable adaptation to swapping of the good arm in a two-armed bandit.}
    \label{fig:bandit_meta}
\end{figure}

\paragraph{Meta Learning a Bandit Learning Algorithm}
Given a non-stationary task, a good policy can not exhibit a fixed behavior but must adapt to changing rewards (learn).
We test the capabilities of {\approach} to adapt to a changing bandit task.
In \autoref{fig:bandit_meta} we swap the good and bad arm at random intervals.
We further feed the reward as an input to the policy such that it can adapt its behavior based on the reward.
We observe that {\approach} leads to self-modifying policies that change their action (learn) in response to the reward they previously received.
In contrast, if this reward is not fed as an input, the policy fails to adapt.

\section{Related Work}

\paragraph{In-context Learning and Fast Weight Programmers}
In-context learning~\citep{hochreiter2001learning,wang2016learning,duan2016rl,brown2020language} (or memory-based meta learning) describes neural networks that learn-to-learn purely within their activations.
This is enabled through the inclusion of a feedback signal or demonstration in the network inputs~\citep{schmidhuber1993self}.
Previous work has demonstrated the capabilities of neural networks to encode human-engineered learning algorithms such as backpropagation~\citep{linnainmaa:1970} purely in the forward-pass of neural networks or to discover more efficient learning algorithms from scratch~\citep{kirsch2020meta}.
Closely related to in-context learning are fast weight programmers (FWPs) that learn to update weights explicitly~\citep{Schmidhuber:92ncfastweights,Schmidhuber:93ratioicann,miconi2018differentiable,schlag2020learning}.
The principles that connect in-context learning with learned weight updates are parameter sharing and multiplicative interactions~\citep{kirsch2020meta,kirsch2021introducing}.
Almost all in-context learners to date are not fully self-referential and require (meta-)parameters to be explicitly meta-optimized by known learning algorithms.

\paragraph{Self-referential Meta Learning}
Self-referential meta learning makes all variables (weights and activations) time-varying and modifiable to the neural network itself~\citep{schmidhuber1993self,irie2021a,flennerhag2021bootstrapped}.
As we have discussed in this paper, both self-referential neural networks as well as in-context learners can in principle implement learning, meta learning, meta meta learning, etc in their neural network dynamics.
This still requires explicit meta optimization of some (initial) parameters.
In contrast, we introduced fitness monotonic execution (FME) to allow self-modifications to govern (meta-)learning without explicit meta optimization.
A related algorithm to this is the success story algorithm (SSA)~\citep{schmidhuber1997shifting}.
In this algorithm, regular checkpoints of the learner are created and self-modifications are reverted when the average reward decreases.
Different from our approach, this process is sequential and keeps only a single history of self-modifications.
This potentially limits parallelizability on modern hardware and results in slower exploration.

\paragraph{Portfolio Algorithms}
Portfolio algorithms~\citep{huberman1997economics} are algorithmically related to FME, although view the selection of the algorithm/solution as the ‘meta learning’ problem~\citep{gagliolo2011algorithm}.
This is different from FME where the meta learning can happen within the algorithms themselves and the aim is to put minimal assumptions/biases into the execution of solutions (portfolio schedule).

\section{Discussion}

\paragraph{Limitations}
This paper represents an initial discussion of self-referential systems that do not rely on fixed human-engineered meta optimization.
The empirical evaluation is still minimalistic at this time but should be a good starting point for larger experiments and improvements.
Our intention is to incite further interest in this research direction.
Semantically, it is an open question whether FME should be called a (minimalistic) meta optimizer after all.
Usually, optimizers define how solutions are modified based on a feedback signal.
In the case of FME, these modifications are self-generated.

\paragraph{Broader Impact}
This paper is not directly concerned with applications of machine learning.
Nevertheless, meta learning methods may pose additional challenges in the future.
For instance, learning algorithms are now also subject to learning which means that learning itself becomes more difficult to interpret and to monitor for biases extracted from data.

\paragraph{Self-modifying Architecture}
We described a meta learner that can self-modify all its variables including those that define the self-modifications, but its architecture is still hard-coded.
While many architectures are computationally universal~\citep{siegelmann1991turing}, modifications may still be useful for efficiency reasons~\citep{miller1989designing,elsken2019neural}.
In fitness monotonic execution, the self-modifications do not require differentiability.
Thus, self-modifications can be extended to include architecture modifications $\phi, y, g \leftarrow g_\phi(x)$, such as adding or removing neurons and weights, changing operations, and (un-)sharing variables.

\paragraph{Communication between Solutions/Agents}
Solutions are executed entirely isolated in the present experiments.
Through fitness monotonic execution, they compete for computational time.
Related to approaches simulating artificial life~\citep{langton1997artificial}, this may be extended to collaboration and other interactions through the addition of communication channels between solutions (agents).

\section{Conclusion}
In this paper, we discussed self-referential systems that exhibit self-improvement while reducing the reliance on human engineering to the largest extent possible.
In particular this means avoiding the use of human engineered learning algorithms on the meta level.
We showed that in order to construct systems that can change all their parameters (or variables more generally), functionality needs to be reused.
This is done in the form of parameter (variable) sharing.
We further demonstrated the representational equivalence between neural networks in-context learning with memory and self-referential architectures while highlighting the benefit of self-referential architectures in the absence of meta optimization.
Finally, we proposed {\approach} (FME), a simple approach to avoid explicit meta optimization.
A neural network self-modifies to solve bandit and classic control tasks, improves its self-modifications, and learns how to learn, purely by assigning more computational resources to better performing solutions.

\section*{Acknowledgements}
This work was supported by the ERC Advanced Grant (no: 742870) and computational resources by the Swiss National Supercomputing Centre (CSCS, projects s978, s1041, and s1127).
We also thank NVIDIA Corporation for donating several DGX machines as part of the Pioneers of AI Research Award, IBM for lending a Minsky machine, and weights \& biases~\citep{biewald2020} for their great experiment tracking software and support.

\bibliography{main}
\bibliographystyle{icml2022}

\newpage
\appendix
\onecolumn

\section{Implementation Details}\label{app:exp_details}

\paragraph{Least-recently-used Buffer}
We initialize a least-recently-used (LRU) buffer with a single randomly initialized neural network.
The $m = 100$ buckets evenly cover the entire current performance range and each holds the $100$ most recent solutions.
Solutions from buckets with higher performance are sampled exponentially more frequently.
We use an exponential base of $e^{20}$.
All layers are initialized from a truncated normal with a standard deviation of $\sigma = \frac{1}{\sqrt{N_x}}$.
When selected, a solution is executed for $L=1000$ steps.

\paragraph{Architecture}
We stack three self-referential layers with $32$ hidden units.

\paragraph{Sources of Randomness}
To create a temporal tree of self-modifying solutions, randomness must be injected into the system.
This randomness originates from the policy action sampling, non-deterministic environment steps, and potential external noise injection as an input to the policy.
We found external noise injection not to improve the agent's performance when sufficient randomness originates from the policy and environment.

\end{document}